\documentclass[5p,times]{elsarticle}

\usepackage{amssymb}
\usepackage{amsmath}
\usepackage{multirow}
\usepackage{booktabs}
\usepackage[normalem]{ulem}
\usepackage{ragged2e}
\usepackage{algorithm}
\usepackage[noend]{algpseudocode} 

\usepackage{subfigure}
\usepackage{etoolbox}
\makeatletter
\expandafter\patchcmd\csname\string\algorithmic\endcsname
  {\itemsep\z@}
  {\itemsep = 1.5pt }%
  {}{}
\makeatother
\usepackage[dvipsnames]{xcolor}
\definecolor{mygray}{gray}{0.6}

\usepackage{amsthm}
\newtheorem{definition}{Definition}
\newtheorem{theorem}{Theorem}
\newtheorem{lemma}{Lemma}

\usepackage{multirow} 
\usepackage{booktabs} 
\usepackage{siunitx} 

\begin{document}
\begin{sloppypar}
\begin{frontmatter}



\title{Privacy-Preserving Federated Deep Clustering based on GAN}


\author[label1]{Jie Yan}

\author[label1]{Jing Liu}

\author[label1]{Ji Qi}

\author[label1]{Zhong-Yuan Zhang\corref{mycorrespondingauthor}}

\address[label1]{School of Statistics and Mathematics, \\ Central University of Finance and Economics, Beijing, P.R.China}

\cortext[mycorrespondingauthor]{Corresponding author}
\ead{zhyuanzh@gmail.com}

\begin{abstract}
Federated clustering (FC) is an essential extension of centralized clustering designed for the federated setting, wherein the challenge lies in constructing a global similarity measure without the need to share private data. Conventional approaches to FC typically adopt extensions of centralized methods, like K-means and fuzzy c-means. However, these methods are susceptible to non-independent-and-identically-distributed (non-IID) data among clients, leading to suboptimal performance, particularly with high-dimensional data.  In this paper, we present a novel approach to address these limitations by proposing a Privacy-Preserving Federated Deep Clustering based on Generative Adversarial Networks (GANs). Each client trains a local generative adversarial network (GAN) locally and uploads the synthetic data to the server. The server applies a deep clustering network on the synthetic data to establish $k$ cluster centroids, which are then downloaded to the clients for cluster assignment. Theoretical analysis demonstrates that the GAN-generated samples, shared among clients, inherently uphold certain privacy guarantees, safeguarding the confidentiality of individual data. Furthermore, extensive experimental evaluations showcase the effectiveness and utility of our proposed method in achieving accurate and privacy-preserving federated clustering.
\end{abstract}


\end{frontmatter}


\section{Introduction}
Clustering, a fundamental task in machine learning, aims to group similar samples, serving as a crucial initial step for various data mining tasks, including domain adaptation \cite{li2021domain}, anomaly detection \cite{markovitz2020graph}, and representation learning \cite{caron2020unsupervised,rezaei2021learning}. Traditionally, clustering is performed in a centralized manner, assuming data stored on a central server for model training. However, in practical scenarios, data may be distributed among numerous client devices, such as smartphones, and can only be kept local at the clients due to privacy restrictions. Figure \ref{toy_case_vis} illustrates that relying solely on local similarity is inadequate for accurately grouping local data, while utilizing global similarity can yield better results. Unfortunately, accessing the global real dataset is unfeasible since sharing local data among clients is strictly prohibited. Thus, the crux of the matter lies in devising a means to measure global similarity without sharing private data.

\begin{figure}[!t]
\centering
\subfigure[Client 1]{
\includegraphics[height = 3cm, width = 4cm]{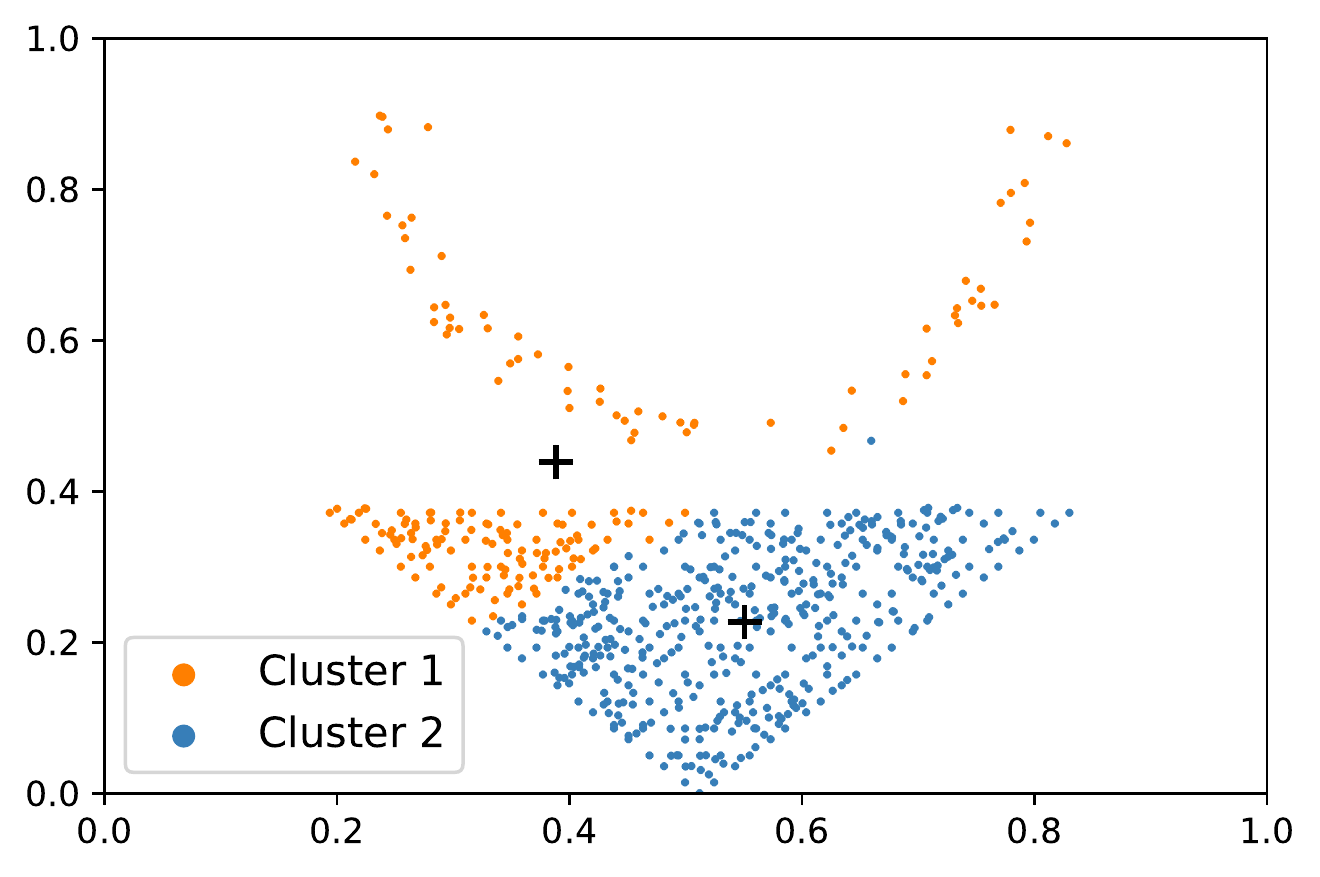}}
\quad
\subfigure[Client 2]{
\includegraphics[height = 3cm, width = 4cm]{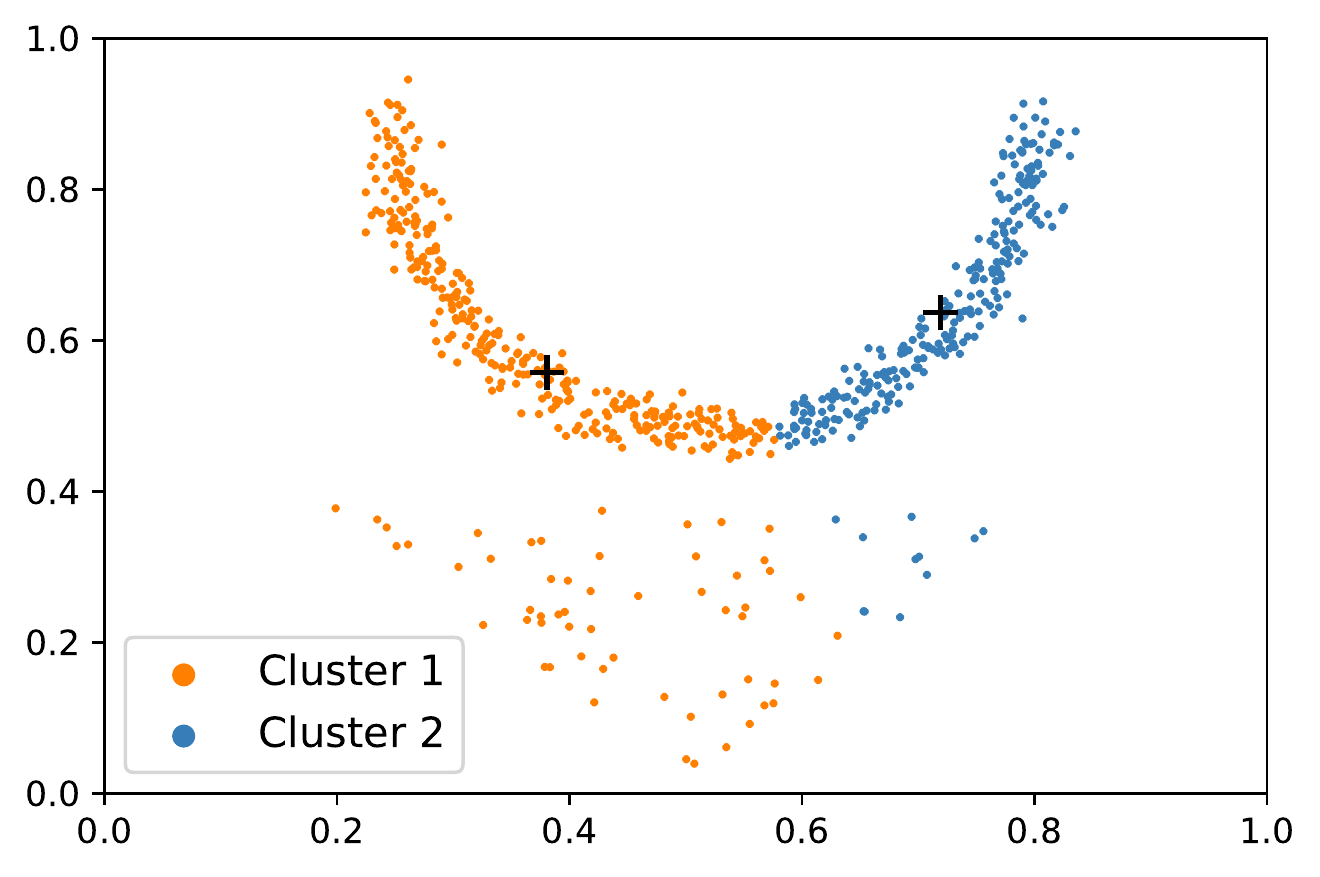}}

\subfigure[Global real dataset]{
\includegraphics[height = 3cm, width = 4cm]{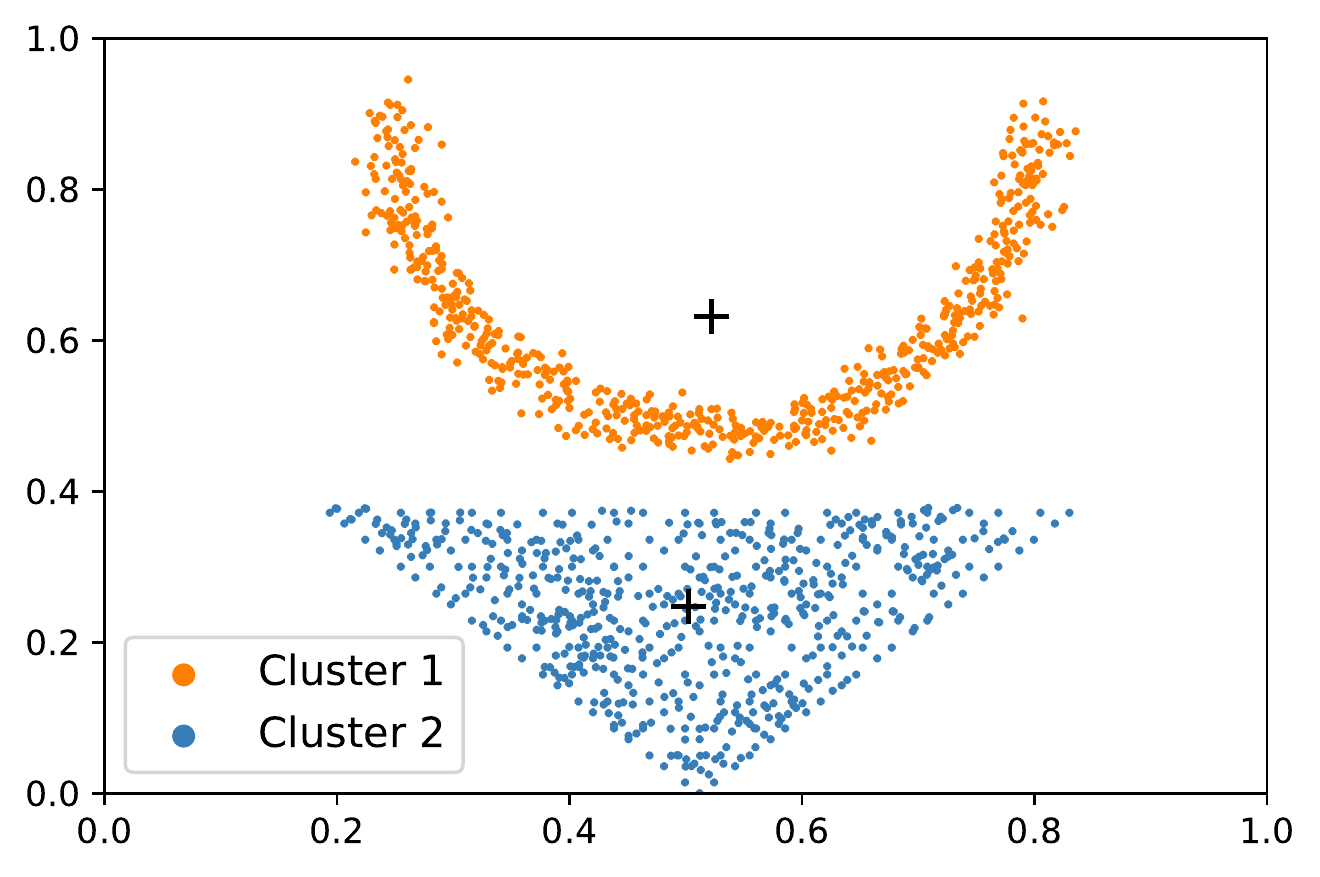}}
\quad
\subfigure[Global synthetic dataset]{
\includegraphics[height = 3cm, width = 4cm]{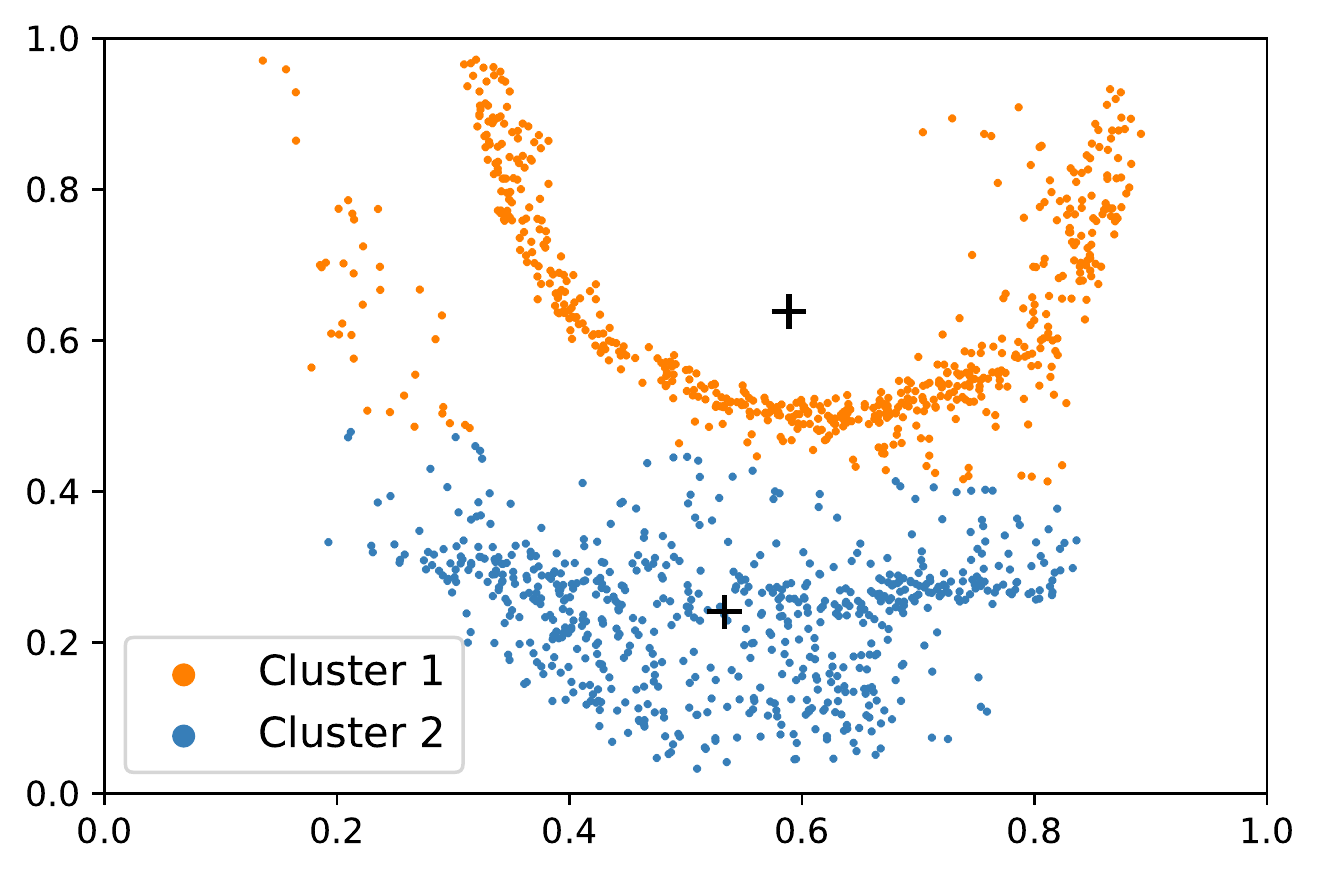}}
\caption{A toy example to illustrate the inherent interest of federated clustering. (a)-(b) The local data on two client devices. (c) All local data are combined in a central storage. (d) GAN generated synthetic data used in our model. For each dataset, we perfromed  K-means on it. Cluster centroids are indicated by "+" and samples are colored by the clustering results. Only (c) correctly identifies two clusters, and the centroids in (d) are the closest to those in (c).}
\label{toy_case_vis}
\end{figure}

To address this challenge, a novel clustering protocol, Federated Clustering (FC), has emerged, aiming to cluster data based on a global similarity measure while preserving data locally \cite{dennis2021heterogeneity, stallmann2022towards}. Extensions of traditional methods, such as K-means (KM) and fuzzy c-means (FCM), have been proposed for FC, referred to as k-FED \cite{dennis2021heterogeneity} and Federated Fuzzy C-Means (FFCM) \cite{stallmann2022towards}. These extensions iteratively estimate global and local data centroids, mining local centroids from private data, and uploading them to a server where KM is applied to create $k$ global cluster centroids for preserving global similarity information. Despite their effectiveness, these approaches suffer from three main limitations. Firstly, the constructed global cluster centroids may be sensitive to varying levels of non-independent-and-identically-distributed (non-IID) data among clients, as the distribution of local data heavily impacts local cluster centroids. Secondly, they assume data follows a Gaussian mixture model, which rarely holds true for real-world datasets in their original data space. Lastly, these algorithms perform poorly on high-dimensional data due to the curse of dimensionality.

To overcome these challenges, we propose a novel framework called \textbf{p}rivacy-\textbf{p}reserving \textbf{f}ederated deep \textbf{c}lustering based on \textbf{GAN} (\textbf{PPFC-GAN}). This approach involves training Generative Adversarial Networks (GANs) \cite{goodfellow2014generative} locally on clients' private data, generating synthetic data that alleviates the non-IID issue. Subsequently, a Deep Clustering Network (DCN) is applied to the synthetic data to perform dimension reduction and KM simultaneously, constructing $k$ global cluster centroids in the latent space. These centroids are then downloaded to the clients, and final cluster assignments are determined based on the distance between local data and centroids. Although sharing synthetic data as a substitute for sharing private data is a natural approach in federated scenarios, it still raises potential privacy concerns \cite{Augenstein2020Generative, chen2020gs, xin2020private}. Therefore, it is crucial to establish the theoretical foundation for the privacy guarantees associated with the synthetic data utilized in our proposed method, as analyzed in Sect. \ref{theorem}.

\begin{table}[!t]
\caption{NMI of clustering methods on the toy example. The results indicate that K-means (KM) and fuzzy c-means (FCM) can recover two true clusters in the centralized setting, but their extensions fail to do so in the federated setting.}
\renewcommand{\arraystretch}{1.25} 
\tabcolsep 4.6mm 
\begin{tabular}{cccccc}
\hline\hline
\multicolumn{2}{c}{Centralized setting} &\multicolumn{3}{c}{Federated setting}\\ \cmidrule(r){1-2} \cmidrule(r){3-5}
KM &FCM & k-FED &FFCM &ours\\
\hline
1 &1  &0.7352 &0.3691 &1\\

\hline\hline
\end{tabular}\label{toy_case_nmi}
\end{table}

The effectiveness of our proposed method is intuitively illustrated by a toy example in Fig. \ref{toy_case_vis}, and the results are presented in Table \ref{toy_case_nmi}, demonstrating that: 1) The synthetic data generated is a good proxy for the real data. 2) Our proposed method exhibits superior performance compared to alternatives. In addition to the aforementioned challenges, federated settings involve other concerns, such as expensive communication, systems heterogeneity, model inefficiency, and device failures. Nevertheless, the proposed method effectively addresses these concerns, requiring only one communication round between the central server and clients, and can run asynchronously while remaining robust to device failures.

In summary, our contributions are as follows: 1) The proposed method significantly outperforms k-FED and FFCM. 2) Many studies \cite{zhu2021federated} have highlighted the adverse effects of non-IID data among clients, and our method resolves this problem by sharing synthetic data generated from local GANs. 3) We theoretically prove that the GAN-generated samples shared by clients inherently satisfy certain privacy guarantees. 4) Systematic experiments reveal that the proposed method is more effective and robust than the baselines in immunizing against non-IID problems and device failures, and can even benefit from some non-IID scenarios.

The rest of this paper is organized as follows: Sect. \ref{related_work} provides an overview of representative methods in centralized and federated clustering, respectively. Following that, Sect. \ref{sect3} introduces some preliminaries about GANs and then presents our new federated deep clustering method. Sect. \ref{sect4} demonstrates the advantages of our proposed method. Finally, Sect. \ref{sect6} concludes this paper.

\begin{figure}[!t]
\centering
\subfigure[DCN architecture]{
\includegraphics[height = 4cm, width = 8cm]{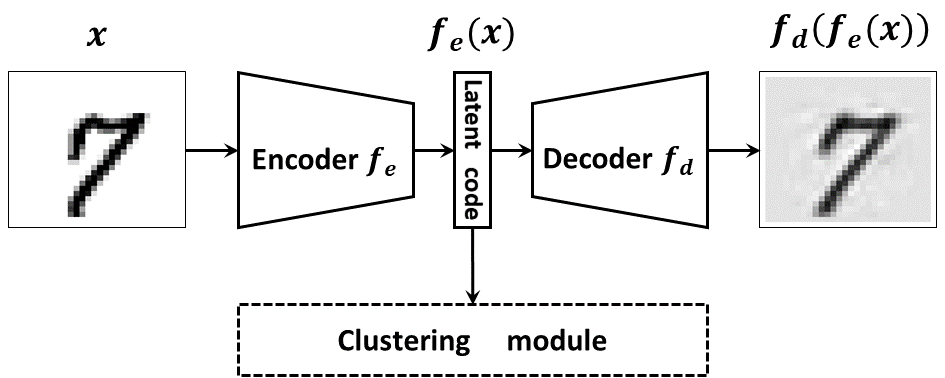}}

\subfigure[Original data space]{
\includegraphics[height = 4cm, width = 4cm]{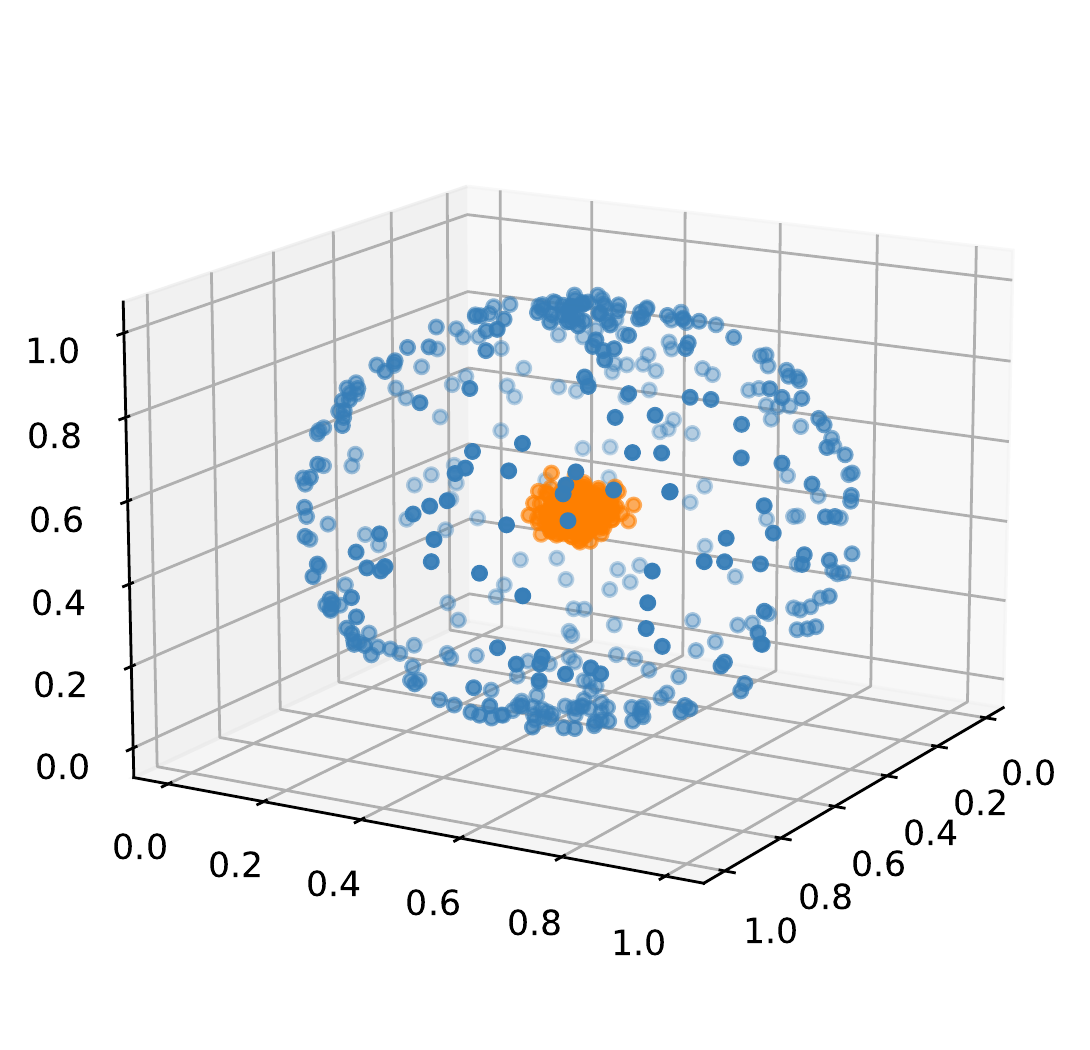}}
\quad
\subfigure[Latent space]{
\includegraphics[height = 3.5cm, width = 4cm]{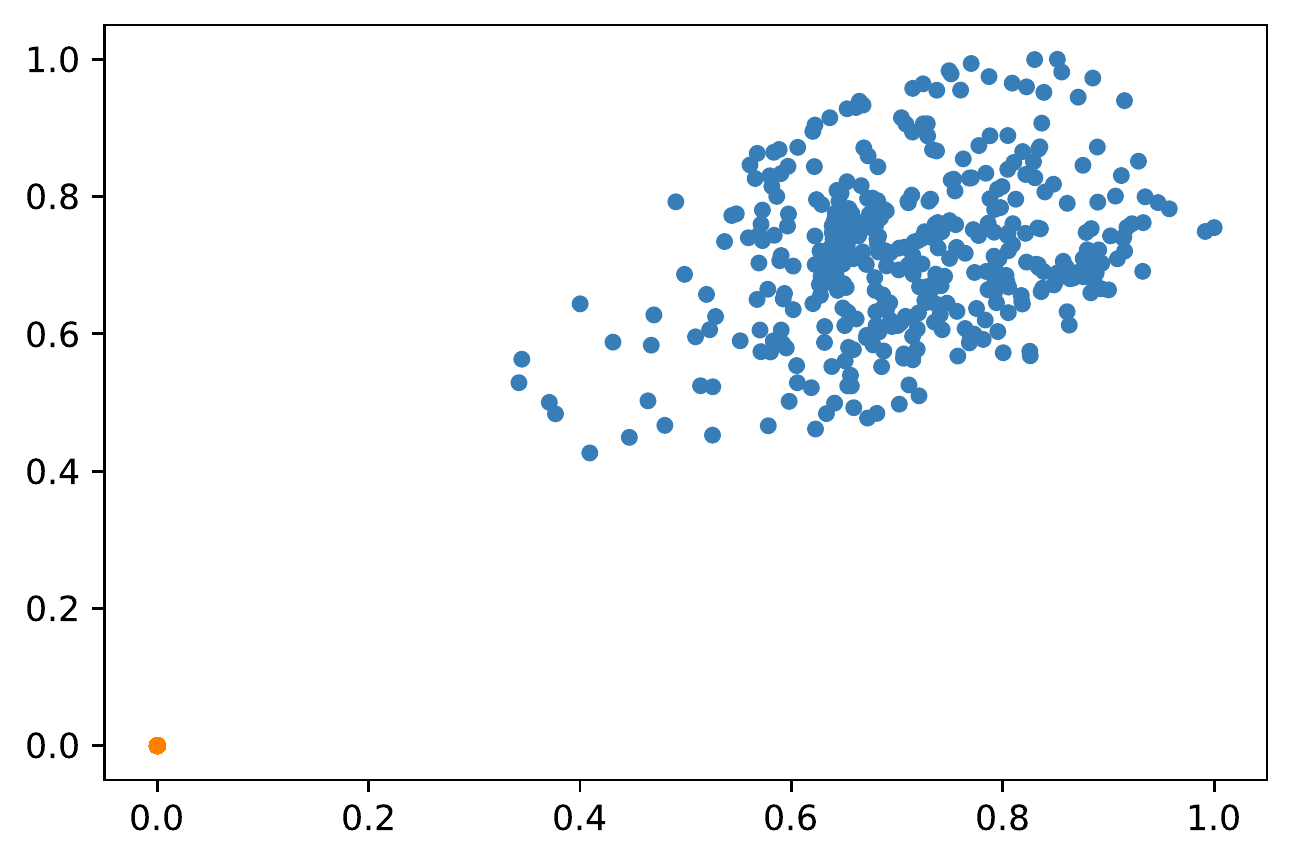}}
\caption{(a) The DCN architecture. (b)-(c) A toy example to illustrate the effectiveness of DCN. Each color corresponds to a cluster. In the original data space, recovering the true clusters using KM is challenging, but in the latent space, it becomes easier to achieve.}
\label{case_1_vis}
\end{figure}

\section{Related Work}\label{related_work}
In this section, we will review some representative methods in centralized clustering and federated clustering, respectively.

\subsection{Centralized clustering}\label{DCN_related}
The prominent approach for centralized clustering is K-means (KM) \cite{macqueen1967classification}, widely used despite its limitations, including the assumption of Gaussian data distribution and inefficiency in high-dimensional spaces due to the curse of dimensionality. To address these challenges, dimension reduction techniques have been employed to transform high-dimensional data into a lower-dimensional latent space, which retains meaningful properties and is more amenable to K-means clustering. Representative methods encompass Principal Component Analysis (PCA), Nonnegative Matrix Factorization (NMF), and Stacked Autoencoders (SAE) \cite{vincent2010stacked}. Recent research demonstrates that combining dimension reduction and clustering methods can enhance performance, leading to the development of Deep Clustering Network (DCN) \cite{yang2017towards}.

As depicted in Fig. \ref{case_1_vis}, DCN comprises two components: an SAE-based dimension reduction module and a clustering module (KM), with alternating parameter optimization. The SAE consists of an encoder, which generates a compressed low-dimensional representation of input data, and a decoder, which reverses the process. An effective SAE within DCN should retain essential information from the input data, resulting in low-dimensional representations more suitable for KM. The objective function of DCN is defined as follows:
\begin{equation}
\mathop{\min}\limits_{f_e,\, f_d,\, C, \left\{ s_i\right\}} \sum_{i=1}^{N}\left( \left\|x_i - f_d(f_e(x_{i}))\right\|_{2}^{2} + \frac{\lambda}{2}\left\|f_e (x_{i}) - C{s}_{i}\right\|_{2}^{2}\right),
\label{DCN}
\end{equation}
where $f_e$ is the encoder that inputs samples and outputs their low-dimensional representations, $f_d$ is the decoder that inputs the low-dimensional representations and outputs the reconstructed samples, $C$ is the representation matrix of cluster centroids and each column of it corresponds to a cluster centroid, $s_i$ is the hard assignment vector of $x_i$ and it is a one-hot vector, and $\lambda$ is a tradeoff hyperparameter.

Although DCN effectively addresses complex data clustering in centralized settings, its potential in federated scenarios remains unexplored. This work aims to extend DCN to federated clustering.

\subsection{Federated clustering}
Federated clustering aims to cluster data based on a global similarity measure while preserving privacy, prohibiting direct measurement of similarity among samples across clients. The primary challenge is to measure global similarity while keeping all data local.

To tackle this challenge, two similar methods, k-FED \cite{dennis2021heterogeneity}, and Federated Fuzzy C-means (FFCM) \cite{stallmann2022towards}, were proposed. In these methods, each client runs a classic centralized clustering method on its local data to generate local cluster centroids, which are then uploaded to the central server. The central server constructs $k$ global cluster centroids by running KM on the uploaded local cluster centroids. The classic centralized clustering method used in k-FED is KM and that used in FFCM is fuzzy c-means (FCM). However, these methods suffer from three main limitations. Firstly, the constructed global cluster centroids may be sensitive to varying levels of non-independent-and-identically-distributed (non-IID) data among clients, as the distribution of local data heavily impacts the local centroids. Secondly, they assume data follows a Gaussian mixture model, rarely satisfied in real-world datasets in their original data space. Lastly, these algorithms perform poorly on high-dimensional data due to the curse of dimensionality.

To address the first limitation, we observe that the non-IID level measures the heterogeneity degree among local data distributions, which is independent of the global distribution. Constructing a good approximation of the global data could potentially immunize the model against the non-IID problem. Additionally, dimensionality reduction methods can transform high-dimensional data into a lower-dimensional space that retains meaningful properties and is more suitable for clustering \cite{yang2017towards}. Building upon these insights, we propose a simple yet effective federated deep clustering framework, \textbf{p}rivacy-\textbf{p}reserving \textbf{f}ederated deep \textbf{c}lustering based on \textbf{GAN} (\textbf{PPFC-GAN}), which extends DCN to federated settings.

\section{Privacy-Preserving Federated Deep Clustering based on GAN (PPFC-GAN)}
\label{sect3}

In this section, we first introduce some preliminaries. Then, we propose a new federated deep clustering framework with GAN-based data synthesis, which is called \textbf{p}rivacy-\textbf{p}reserving \textbf{f}ederated deep \textbf{c}lustering based on \textbf{GAN} (\textbf{PPFC-GAN}).

\subsection{Preliminaries}
\subsubsection{Generative adversarial network (GAN)}
Generative Adversarial Networks (GANs) have proven highly successful in diverse generative tasks, including image generation \cite{liu2021blendgan}, image super-resolution \cite{shi2022latent}, and image completion \cite{liu2021pd}. A vanilla GAN consists of two networks: the generator and discriminator. The generator aims to produce synthetic samples that deceive the discriminator, which, in turn, strives to differentiate between synthetic and real samples. The training process concludes when the discriminator can no longer distinguish between the two, indicating that the generator has approximated the real data distribution, achieving the theoretical global optimum. The GAN's objective function is defined as follows:
\begin{equation}
\mathop{\min}\limits_{G}\mathop{\max}\limits_{D}\mathop{\mathbf{E}}
\limits_{z \sim \mathcal{N}}\log(1 - D(G(z))) + \mathop{\mathbf{E}} \limits_{x \sim p_r}\log(D(x)),
\end{equation}
where $G$ is the generator that inputs a noise $z$ and outputs a synthetic sample, $\mathcal{N}$ is Gaussian distribution, $D$ is the discriminator that inputs a sample and outputs a scalar to tell the synthetic samples from the real ones, and $p_r$ is the distribution of real data.

\begin{figure}[!t]
\centering
\subfigure[$z \sim$ $\mathcal{N}\left(0, \, I\right)$]{
\includegraphics[height = 3cm, width = 4cm]{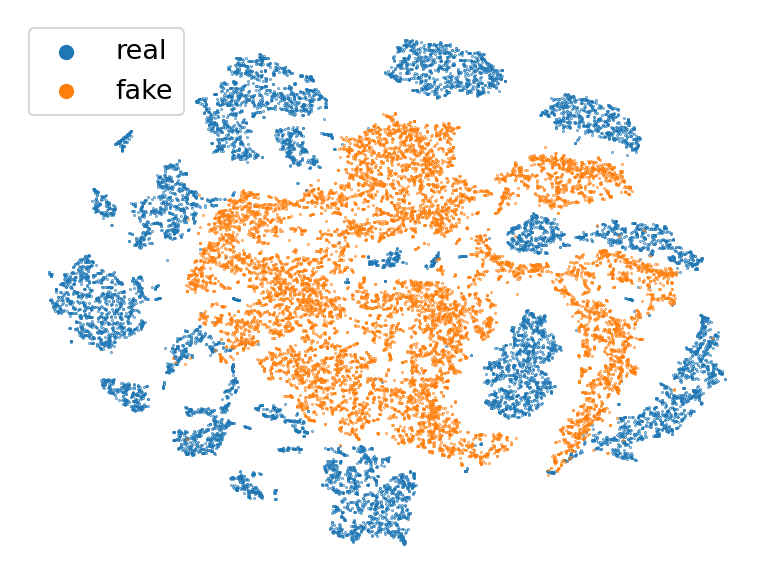}}
\quad
\subfigure[$z \sim (z_n,\, z_c)$ ]{
\includegraphics[height = 3cm, width = 4cm]{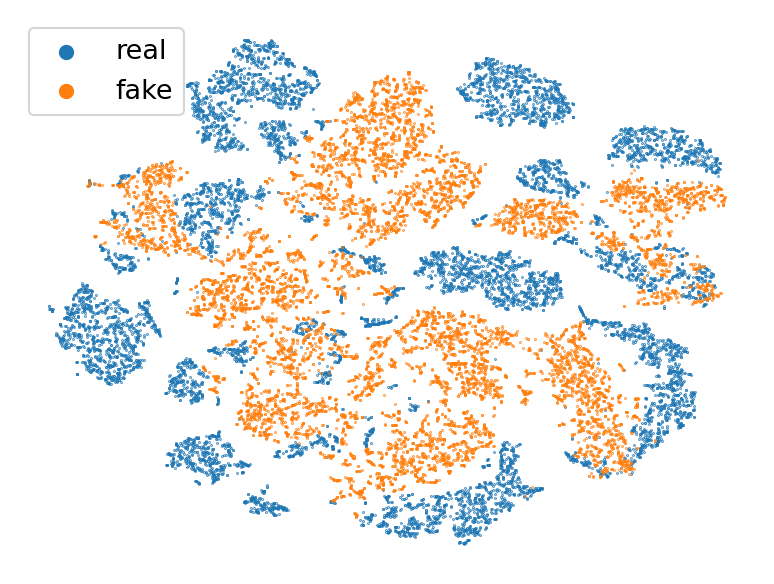}}
\caption{t-SNE visualization of the synthetic data for GANs trained with different priors on Pendigits. The mixture prior can alleviate the mode collapse problem.}
\label{priors}
\end{figure}

However, practical GAN training encounters challenges, particularly the well-known problem of unstable adversarial training, leading to mode collapses \cite{Luke2017}. In mode collapses, synthetic samples exhibit high quality but low diversity, capturing only a fraction of real data characteristics.

To address mode collapses, the method introduced in \cite{mukherjee2019clustergan} incorporated an additional categorical variable into the generator's input, resulting in synthetic data exhibiting a clearer cluster structure and increased diversity (Fig. \ref{priors}). In our approach, we also adopt a combination of both discrete categorical variables and continuous random variables as input to the generator, with the goal of mitigating mode collapses. The revised GAN objective function is defined as:
\begin{equation}
\mathop{\min}\limits_{G}\mathop{\max} \limits_{D}\mathop{\mathbf{E}}
\limits_{z \sim (z_n,\, z_c)} \log(1 - D(G(z))) + \mathop{\mathbf{E}} \limits_{x \sim p_r}\log(D(x)),
\label{GAN}
\end{equation}
where $z_n \sim$ $\mathcal{N}\left(0, \, I\right)$, \,$z_c = e_u, \,u \sim \mathcal{U}\{1,\, 2,\, \cdots,\, k\}$, and $e_u$ is a one-hot vector with the $u$-th element being 1.

\subsection{Privacy-Preserving Federated Deep Clustering based on GAN (PPFC-GAN)}
To extend the Deep Clustering Network (DCN) to the federated setting, a straightforward integration of DCN into the Federated Averaging (FedAvg) \cite{mcmahan2017communication} framework involves averaging parameters from local DCN models trained on client devices. However, this approach may prove ineffective and communication-inefficient due to the non-IID problem.

To address these concerns, our key insight is that for a given federated dataset, the non-IID level quantifies the heterogeneity degree among local data distributions and remains independent of the global distribution. Therefore, a good approximation of the global data may immunize the model against the non-IID problem. Inspired by this, we propose a simple yet effective federated deep clustering framework, \textbf{p}rivacy-\textbf{p}reserving \textbf{f}ederated deep \textbf{c}lustering based on \textbf{GAN} (\textbf{PPFC-GAN}), extending DCN to federated settings.

PPFC-GAN requires only one round of communication between clients and the central server and consists of two main steps: global synthetic data construction and cluster assignment. The details are as follows:

\subsubsection{Global synthetic data construction}
Given a real-world dataset $X$ distributed among $m$ clients, i.e., $X=\bigcup_{i=1}^{m} X^{(i)}$, each client $i$ ($i = 1,\, 2,\, \cdots, \, m$) downloads an initial GAN model from the central server and trains it with their local data $X^{(i)}$. Subsequently, each client $i$ utilizes the trained generator $G^{(i)}$ to generate a dataset $\hat{X}^{(i)}$ of the same size as $X^{(i)}$ and upload the generated dataset to the central server. Finally, the global synthetic dataset $\hat{X}$ is obtained by merging all generated datasets, i.e., $\hat{X}=\bigcup_{i=1}^{m} \hat{X}^{(i)}$.

\subsubsection{Cluster assignment}
The central server trains a DCN model with the global synthetic dataset $\hat{X}$. It then provides the trained encoder $f_e$ and the learned representation matrix of cluster centroids $C$ to each client. Each local data point can be labeled by solving the optimization problem:
\begin{equation}
\underset{i=\{1,\, \cdots,\, k\}}{\arg \min }\left\|f_e(x)-{c}_{i}\right\|_{2},
\end{equation}
where $c_i$ is the $i$-th column of $C$, representing the $i$-th cluster centroid.

\subsubsection{Theoretical analysis}
\label{theorem}
In federated scenarios, although sharing synthetic data as a substitute for sharing private data is a natural approach, it may raise privacy concerns \cite{Augenstein2020Generative, chen2020gs, xin2020private}. Therefore, establishing the theoretical foundation for the privacy guarantees associated with synthetic data used in our proposed method is crucial.

We start with the formal privacy definition \cite{dwork2008differential, dwork2014algorithmic} and a lemma demonstrating privacy guarantees of GAN-generated samples in centralized scenarios \cite{dwork2008differential, dwork2014algorithmic}. Building upon this, we analyze the privacy guarantees of the synthetic data in the proposed method through Theorem \ref{thm}, affirming that GAN-generated samples shared by clients inherently satisfy privacy guarantees.



\begin{definition}
(Differential privacy \cite{dwork2008differential, dwork2014algorithmic}) If two datasets, $X_0$ and $X_1$, differ in only one sample, we refer to them as neighboring datasets. A mechanism $M$ gives $(\epsilon, \delta)$-differential privacy if for any neighboring datasets $X_0$ and $X_1$, and any set $S \subseteq range(M)$,
\begin{equation}
Pr\left[M\left(X_0\right) \in S\right] \leq e^\epsilon Pr\left[M\left(X_1\right) \in S\right] + \delta. \nonumber
\end{equation}
$\epsilon$ refers to the privacy budget and governs the level of protection and the amount of noise introduced. $\delta$ represents the probability of violating DP constraints.
\end{definition}


\begin{lemma}
\cite{lin2021privacy} Given a GAN trained on $n$ samples and used to generate $s$ samples, the generated samples guarantee $(\epsilon, \delta)$-differential privacy, where $\delta$ scales as $O(s/n)$.
\end{lemma}



\begin{theorem}
The GAN-generated samples shared by client $i$ guarantee $(\epsilon_i, \delta_i)$-differential privacy, where $\delta_i$ scales as $O(s_i / n_i)$, $s_i$ is the number of the generated samples and $n_i$ is the size of local data of client $i$.
\label{thm}
\end{theorem}

\begin{proof}
Since each GAN is trained independently on individual clients, the privacy assurance of a client is solely tied to the generated samples they share with the server. According to Lemma 1, the GAN-generated samples shared by client $i$ guarantee $(\epsilon_i, \delta_i)$-differential privacy, where $\delta_i$ scales as $O(s_i / n_i)$.
\end{proof}

\section{Experimental results}
\label{sect4}
In this section, we first detail the experimental settings. Then, we showcase the effectiveness of PPFC-GAN on several datasets with different non-IID scenarios, and analyze the necessity of simultaneous dimensionality reduction and clustering in the proposed method, both quantitatively and qualitatively. Finally, we validate the sensitivity of different federated clustering methods to device failures induced by system heterogeneity, and summarize the experimental results.

\begin{table}[!t]
\caption{Description of datasets.}
\scalebox{0.7}{
\renewcommand{\arraystretch}{1.5} 
\tabcolsep 3.75mm 
\begin{tabular}{ccccc}
\hline\hline
\textbf{Dataset} & \textbf{Type} &\textbf{Size} & \textbf{Image size/Features} & \textbf{Class} \\\hline
MNIST           & Gray image        & 70000    & $28\times28$   & 10\\
Fashion-MNIST   & Gray image        & 70000    & $28\times28$   & 10\\
CIFAR-10        & RGB image        & 60000    & $32\times32$            & 10 \\
STL-10          & RGB image        & 13000    & $96\times96$           & 10 \\
Pendigits       & Time series  & 10992    & 16             & 10\\

\hline\hline
\label{datasets}
\end{tabular}}
\end{table}

\subsection{Experimental Settings}
There is still a lack of universal non-IID benchmark datasets for FL due to the complexity of federated learning itself \cite{li2022federated,hu2022oarf}. In this paper, following ref. \cite{chung2022federated}, we simulate different federated scenarios by dividing a real-world dataset into $k$ smaller subsets, with each subset corresponding to a specific client, and scaling \textbf{the non-IID levels} through the parameter $p$ for each client, where $k$ is the number of true clusters. For the $i$-th client with $s$ data samples, there are $p\cdot s$ ones sampled from the $i$-th cluster, while the remaining ones are sampled from the entire data. Specially, $p = 0$ means the data are randomly distributed on the clients, whereas $p = 1$ means each client is one cluster.

As is shown in Table \ref{datasets}, four image datasets MNIST, Fashion-MNIST, CIFAR-10 \cite{krizhevsky2009learning} and STL-10 \footnote{Note that, to reduce the computational cost of baseline methods and to use the same network structure for CIFAR-10, we performed a preprocessing step to resize the images in STL-10 to 32 $\times$ 32.} \cite{coates2011analysis}, and a time series dataset Pendigits \cite{keller2012hics} are selected for comprehensive analysis. In PPFC-GAN, all networks are trained with Adam Optimizer \cite{kingma2014adam}. Moreover, to avoid unrealistic tuning, we use the same stacked autoencoder (SAE) architecture in all experiments. The SAE architecture is very simple and the forward network of it has only 3 hidden layers which have 500, 500, 2000 neurons, respectively. The reconstruction network has a symmetric structure and the code layer has 10 neurons. More detailed hyperparameter settings can be found in the Appendix. Codes are available upon request and will be public available after acceptence.

\begin{table}[!t]
\caption{NMI of clustering methods in different federated scenarios. For each comparison, the best result is highlighted in boldface.}
\scalebox{0.6}{
\renewcommand{\arraystretch}{1.25} 
\tabcolsep 1.85mm 
\begin{tabular}{ccccccccc}
\hline\hline

\multirow{2}{*}{Dataset} &\multirow{2}{*}{$p$} &\multicolumn{3}{c}{Centralized setting} &\multicolumn{4}{c}{Federated setting}\\ \cmidrule(r){3-5} \cmidrule(r){6-9}
\quad &\quad &\textcolor{mygray}{KM} &\textcolor{mygray}{FCM} &\textcolor{mygray}{DCN} &k-FED &FFCM &PPFC-GAN$^\dag$ &PPFC-GAN\\

\hline
\multirow{5}{*}{MNIST} &0.0 &\multirow{5}{*}{\textcolor{mygray}{0.5304}} &\multirow{5}{*}{\textcolor{mygray}{0.5187}} &\multirow{5}{*}{\textcolor{mygray}{0.8009}} &0.5081 &0.5157 &0.6026 &\textbf{0.6582}\\
\quad &0.25 &\quad &\quad &\quad &0.4879 &0.5264 &0.5883 &\textbf{0.6392}\\
\quad &0.5 &\quad &\quad &\quad &0.4515 &0.4693 &0.6065 &\textbf{0.6721}\\
\quad &0.75 &\quad &\quad &\quad &0.4552 &0.4855 &0.6657 &\textbf{0.7433}\\
\quad &1.0 &\quad &\quad &\quad &0.4142 &0.5372 &0.7572 &\textbf{0.8353}\\

\hline
\multirow{5}{*}{Fashion-MNIST} &0.0 &\multirow{5}{*}{\textcolor{mygray}{0.6070}} &\multirow{5}{*}{\textcolor{mygray}{0.6026}} &\multirow{5}{*}{\textcolor{mygray}{0.6391}} &0.5932 &0.5786 &0.5725 &\textbf{0.6091}\\
\quad &0.25 &\quad &\quad &\quad &0.5730 &\textbf{0.5995} &0.5519 &0.5975\\
\quad &0.5 &\quad &\quad &\quad &0.6143 &\textbf{0.6173} &0.5384 &0.5784\\
\quad &0.75 &\quad &\quad &\quad &0.5237 &\textbf{0.6139} &0.5696 &0.6103\\
\quad &1.0 &\quad &\quad &\quad &0.5452 &0.5855 &0.6255 &\textbf{0.6467}\\

\hline
\multirow{5}{*}{CIFAR-10} &0.0 &\multirow{5}{*}{\textcolor{mygray}{0.0871}} &\multirow{5}{*}{\textcolor{mygray}{0.0823}} &\multirow{5}{*}{\textcolor{mygray}{0.1260}} &0.0820 &0.0812 &0.1151 &\textbf{0.1165}\\
\quad &0.25 &\quad &\quad &\quad &0.0866 &0.0832 &0.1166 &\textbf{0.1185}\\
\quad &0.5 &\quad &\quad &\quad &0.0885 &0.0870 &0.1185 &\textbf{0.1237}\\
\quad &0.75 &\quad &\quad &\quad &0.0818 &0.0842 &\textbf{0.1157} &0.1157\\
\quad &1.0 &\quad &\quad &\quad &0.0881 &0.0832 &\textbf{0.1337} &0.1318\\

\hline
\multirow{5}{*}{STL-10} &0.0 &\multirow{5}{*}{\textcolor{mygray}{0.1532}} &\multirow{5}{*}{\textcolor{mygray}{0.1469}} &\multirow{5}{*}{\textcolor{mygray}{0.1718}} &\textbf{0.1468} &0.1436 &0.1318 &0.1318\\
\quad &0.25 &\quad &\quad &\quad &0.1472 &0.1493 &0.1449 &\textbf{0.1501}\\
\quad &0.5 &\quad &\quad &\quad &\textbf{0.1495} &0.1334 &0.1469 &0.1432\\
\quad &0.75 &\quad &\quad &\quad &0.1455 &0.1304 &0.1545 &\textbf{0.1590}\\
\quad &1.0 &\quad &\quad &\quad &0.1403 &0.1565 &0.1588 &\textbf{0.1629}\\

\hline
\multirow{5}{*}{Pendigits} &0.0 &\multirow{5}{*}{\textcolor{mygray}{0.6877}} &\multirow{5}{*}{\textcolor{mygray}{0.6862}} &\multirow{5}{*}{\textcolor{mygray}{0.7409}} &0.7001 &0.6866 &0.6812 &\textbf{0.7179}\\
\quad &0.25 &\quad &\quad &\quad &0.6620 &0.6848 &0.6618 &\textbf{0.7054}\\
\quad &0.5 &\quad &\quad &\quad &0.6625 &0.6798 &0.6852 &\textbf{0.7161}\\
\quad &0.75 &\quad &\quad &\quad &0.5521 &0.6757 &0.7057 &\textbf{0.7472}\\
\quad &1.0 &\quad &\quad &\quad &0.6296 &\textbf{0.7236} &0.5927 &0.5627\\
\hline
count &- &- &- &- &2 &4 &2 &17\\
\hline\hline
\end{tabular}\label{NMI}}
\end{table}

\begin{table}[!t]
\caption{Kappa of clustering methods in different federated scenarios. For each comparison, the best result is highlighted in boldface.}
\scalebox{0.6}{
\renewcommand{\arraystretch}{1.25} 
\tabcolsep 1.85mm 
\begin{tabular}{ccccccccc}
\hline\hline

\multirow{2}{*}{Dataset} &\multirow{2}{*}{$p$} &\multicolumn{3}{c}{Centralized setting} &\multicolumn{4}{c}{Federated setting}\\ \cmidrule(r){3-5} \cmidrule(r){6-9}
\quad &\quad &\textcolor{mygray}{KM} &\textcolor{mygray}{FCM} &\textcolor{mygray}{DCN} &k-FED &FFCM &PPFC-GAN$^\dag$ &PPFC-GAN\\

\hline
\multirow{5}{*}{MNIST} &0.0 &\multirow{5}{*}{\textcolor{mygray}{0.4786}} &\multirow{5}{*}{\textcolor{mygray}{0.5024}} &\multirow{5}{*}{\textcolor{mygray}{0.7699}} &0.5026 &0.5060 &0.6065 &\textbf{0.6134}\\
\quad &0.25 &\quad &\quad &\quad &0.4000 &0.5105 &\textbf{0.5848} &0.5773\\
\quad &0.5 &\quad &\quad &\quad &0.3636 &0.3972 &0.5862 &\textbf{0.6007}\\
\quad &0.75 &\quad &\quad &\quad &0.3558 &0.4543 &0.6508 &\textbf{0.6892}\\
\quad &1.0 &\quad &\quad &\quad &0.3386 &0.5103 &0.7480 &\textbf{0.7884}\\

\hline
\multirow{5}{*}{Fashion-MNIST} &0.0 &\multirow{5}{*}{\textcolor{mygray}{0.4778}} &\multirow{5}{*}{\textcolor{mygray}{0.5212}} &\multirow{5}{*}{\textcolor{mygray}{0.5186}} &0.4657 &\textbf{0.4974} &0.4918 &0.4857\\
\quad &0.25 &\quad &\quad &\quad &\textbf{0.5222} &0.5180 &0.4380 &0.4721\\
\quad &0.5 &\quad &\quad &\quad &0.4951 &\textbf{0.4974} &0.4336 &0.4552\\
\quad &0.75 &\quad &\quad &\quad &0.4240 &\textbf{0.4995} &0.4625 &0.4774\\
\quad &1.0 &\quad &\quad &\quad &0.3923 &0.4672 &\textbf{0.5794} &0.5745\\

\hline
\multirow{5}{*}{CIFAR-10} &0.0 &\multirow{5}{*}{\textcolor{mygray}{0.1347}} &\multirow{5}{*}{\textcolor{mygray}{0.1437}} &\multirow{5}{*}{\textcolor{mygray}{0.1599}} &0.1305 &0.1439 &\textbf{0.1488} &0.1426\\
\quad &0.25 &\quad &\quad &\quad &0.1366 &\textbf{0.1491} &0.1458 &0.1400\\
\quad &0.5 &\quad &\quad &\quad &0.1252 &0.1316 &0.1422 &\textbf{0.1443}\\
\quad &0.75 &\quad &\quad &\quad &0.1303 &0.1197 &\textbf{0.1412} &0.1358\\
\quad &1.0 &\quad &\quad &\quad &0.1147 &0.1237 &\textbf{0.1612} &0.1499\\

\hline
\multirow{5}{*}{STL-10} &0.0 &\multirow{5}{*}{\textcolor{mygray}{0.1550}} &\multirow{5}{*}{\textcolor{mygray}{0.1602}} &\multirow{5}{*}{\textcolor{mygray}{0.1909}} &0.1390 &0.1514 &\textbf{0.1579} &0.1557\\
\quad &0.25 &\quad &\quad &\quad &0.1361 &0.1479 &0.1578 &\textbf{0.1611}\\
\quad &0.5 &\quad &\quad &\quad &0.1505 &0.1112 &\textbf{0.1695} &0.1415\\
\quad &0.75 &\quad &\quad &\quad &0.1256 &0.1001 &0.1762 &\textbf{0.1813}\\
\quad &1.0 &\quad &\quad &\quad &0.1328 &0.1351 &0.1832 &\textbf{0.1868}\\

\hline
\multirow{5}{*}{Pendigits} &0.0 &\multirow{5}{*}{\textcolor{mygray}{0.6523}} &\multirow{5}{*}{\textcolor{mygray}{0.6521}} &\multirow{5}{*}{\textcolor{mygray}{0.7489}} &0.7079 &0.6523 &\textbf{0.7403} &0.6966\\
\quad &0.25 &\quad &\quad &\quad &0.6420 &0.6535 &0.7283 &\textbf{0.7466}\\
\quad &0.5 &\quad &\quad &\quad &0.6285 &0.6823 &\textbf{0.7277} &0.6916\\
\quad &0.75 &\quad &\quad &\quad &0.4493 &0.6323 &0.7542 &\textbf{0.7774}\\
\quad &1.0 &\quad &\quad &\quad &0.5222 &\textbf{0.6772} &0.5724 &0.3929\\
\hline

count &- &- &- &- &1 &5 &9 &10\\
\hline\hline
\end{tabular}\label{Kappa}}
\end{table}

\begin{figure}[!t]
\centering
\includegraphics[height = 5cm, width = 8cm]{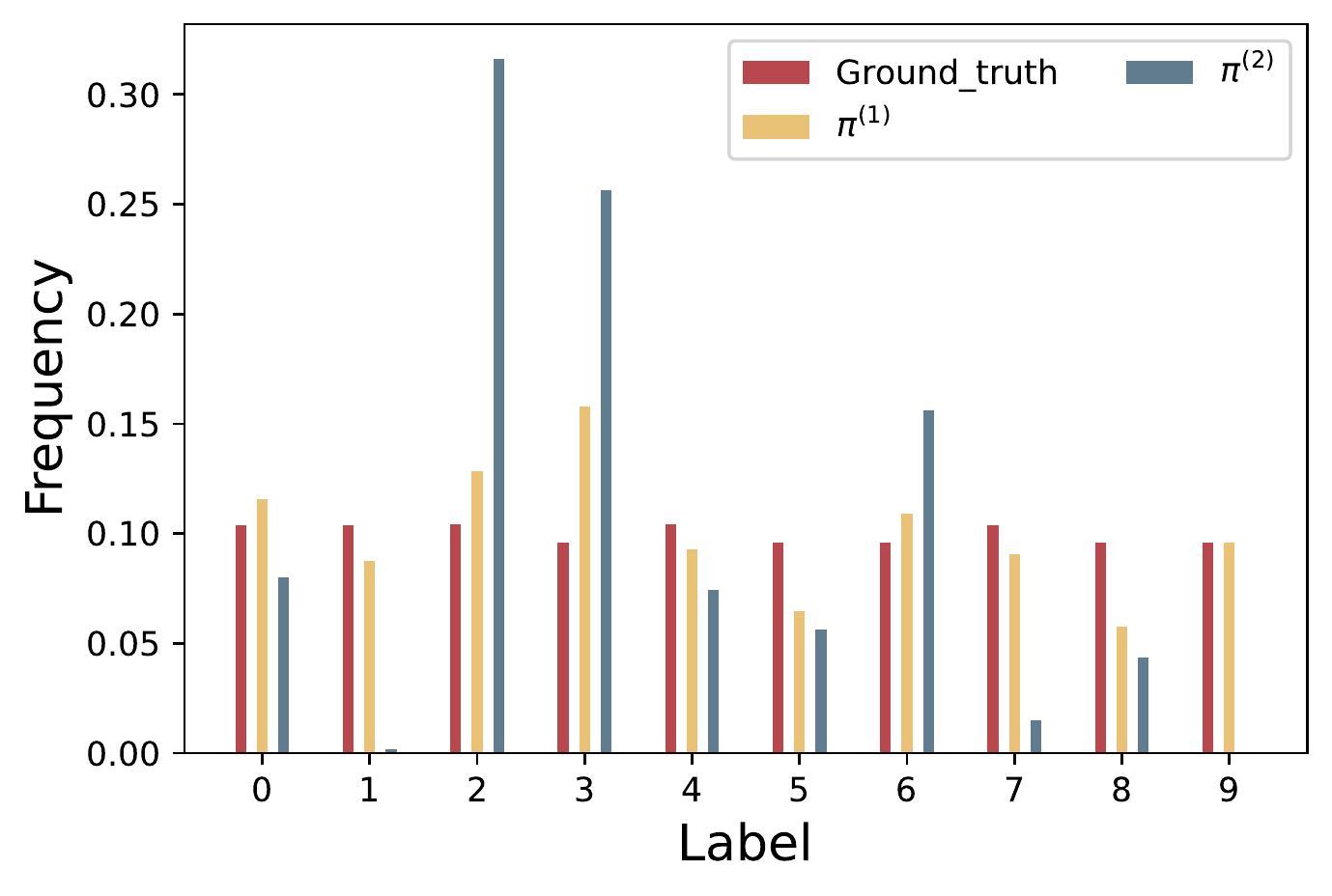}
\caption{The label distribution of different partitions on the dataset Pendigits with $p = 0$. $\pi^{(1)}$ and $\pi^{(2)}$ are two clustering results obtained by PPFC-GAN$^\dag$ and PPFC-GAN, respectively.}
\label{case_kappa}
\end{figure}

\begin{figure*}[!t]
\centering
\subfigure[Global real dataset]{
\includegraphics[height = 6cm, width = 6.4cm]{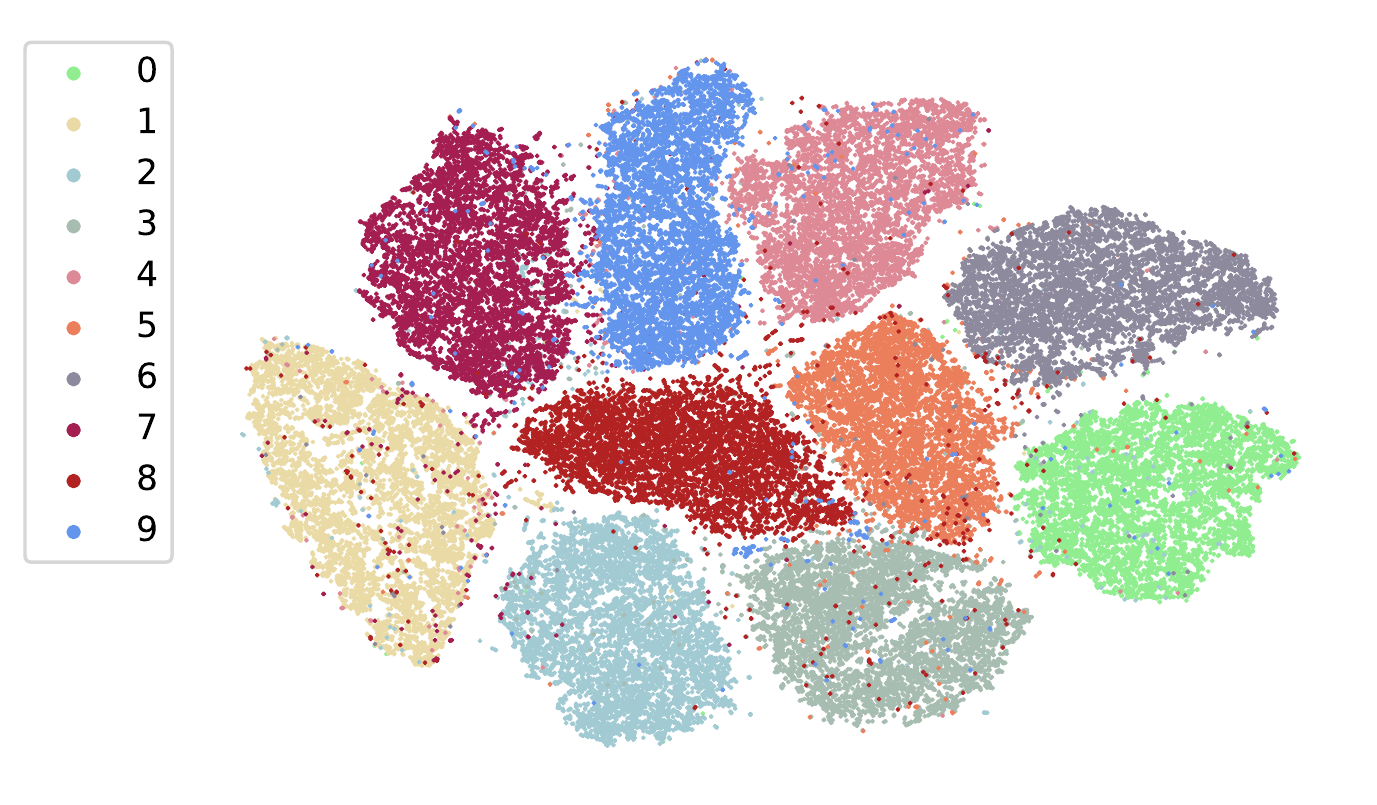}}
\quad
\subfigure[Global synthetic dataset]{
\includegraphics[height = 6cm, width = 6.4cm]{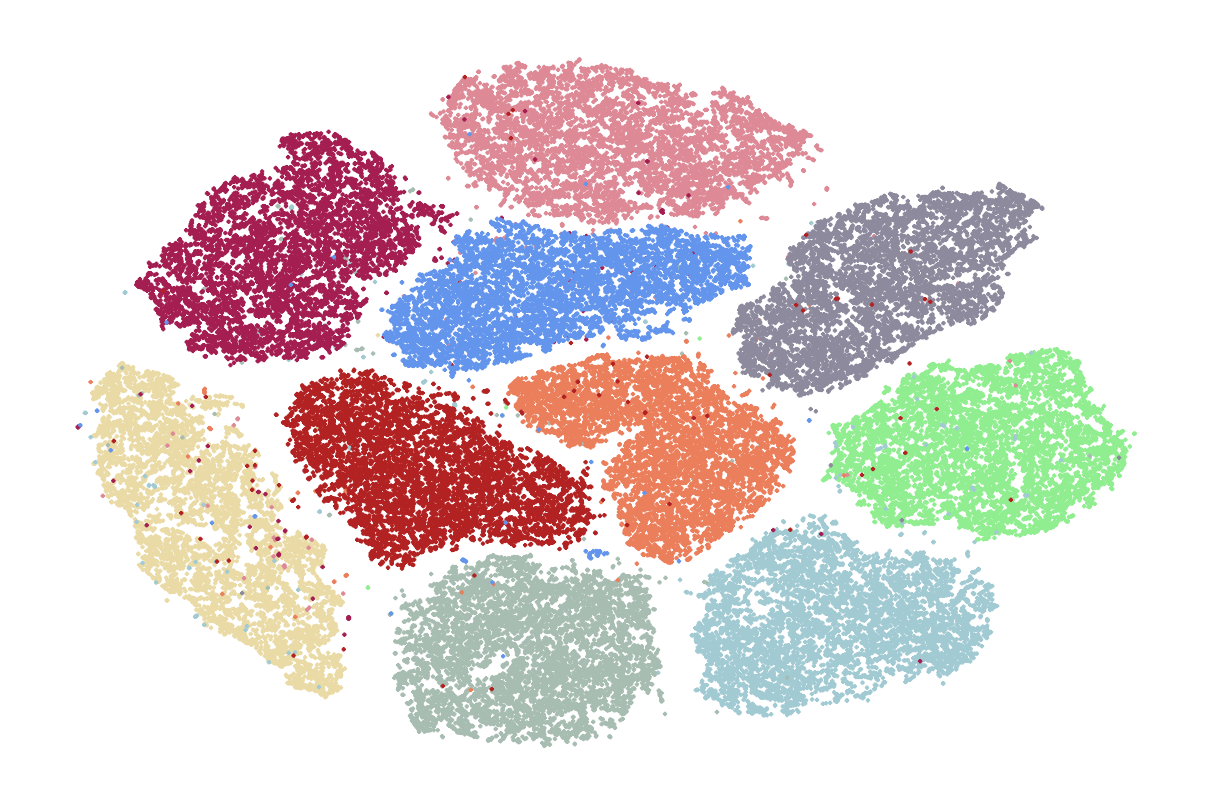}}

\subfigure[PPFC-GAN$^\dag$ result]{
\includegraphics[height = 6cm, width = 6.4cm]{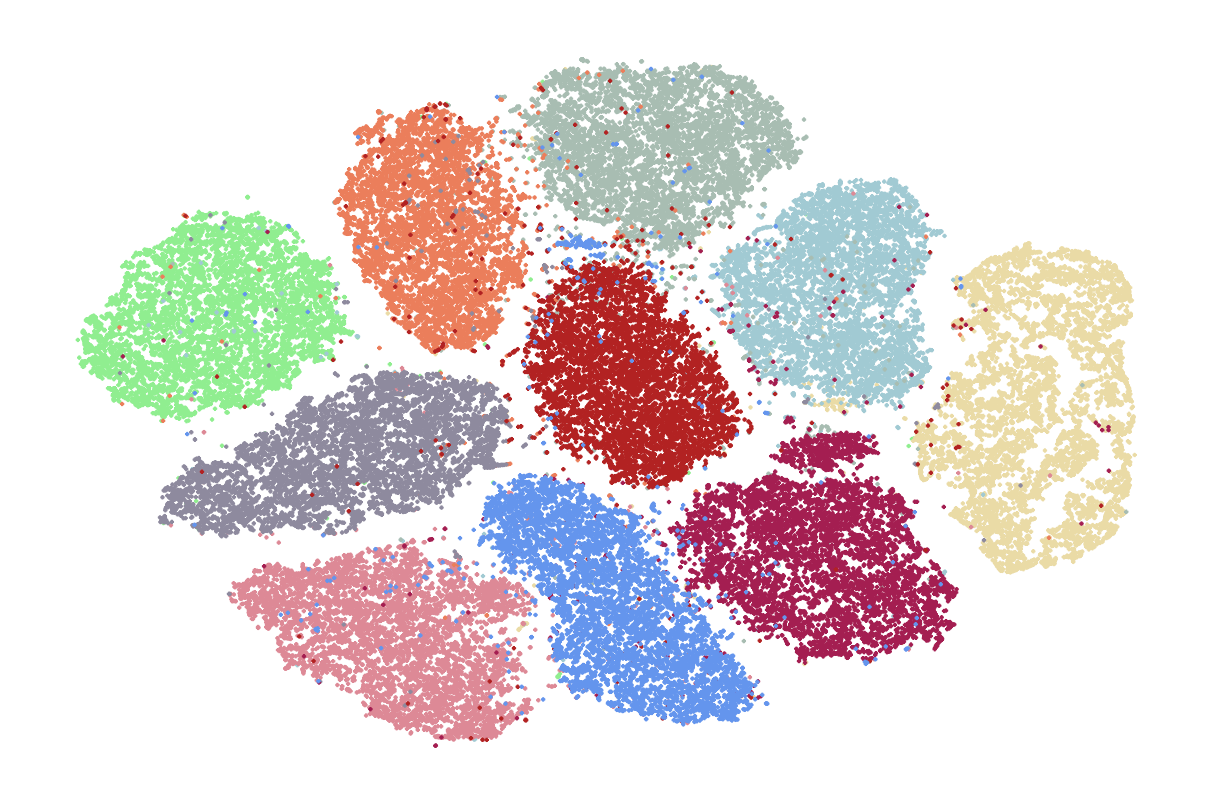}}
\quad
\subfigure[PPFC-GAN result]{
\includegraphics[height = 6cm, width = 6.4cm]{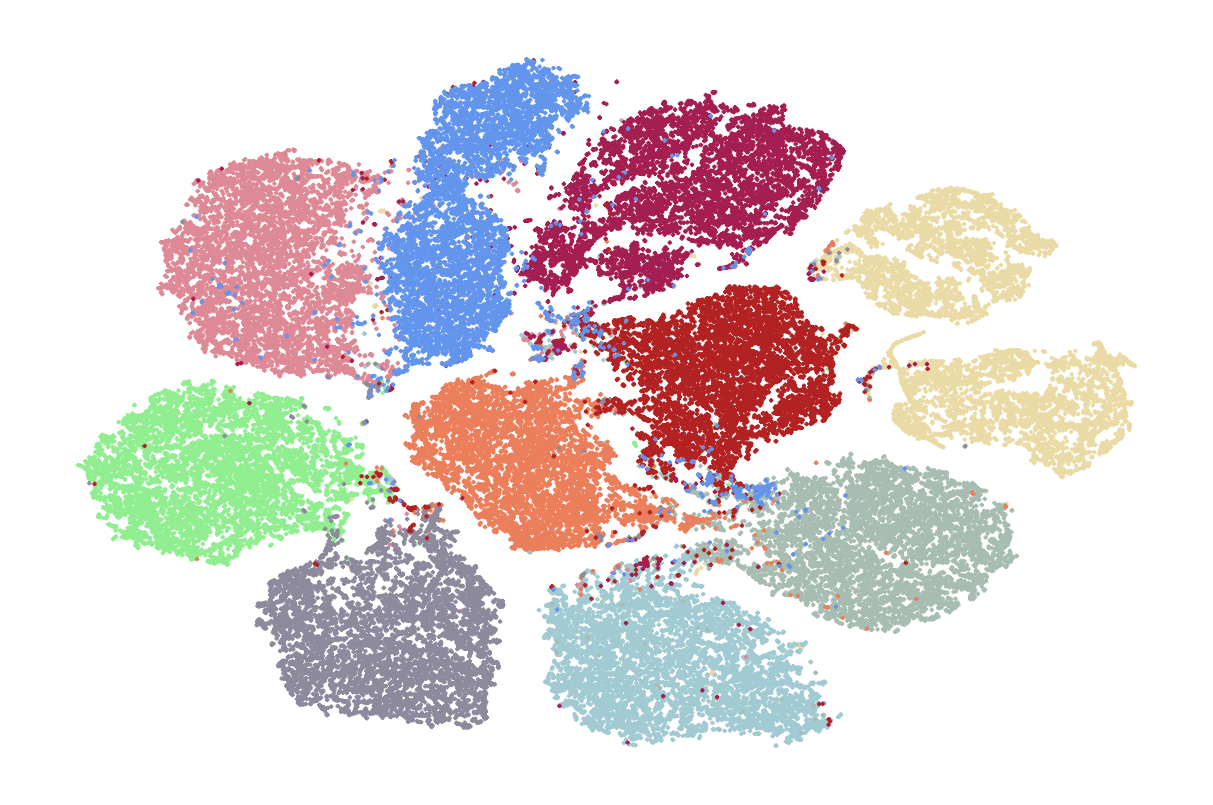}}
\caption{t-SNE visualization on the dataset MNIST with the non-IID level $p = 1$. Each color corresponds to a digit in MNIST.}
\label{tsne_latent}
\end{figure*}

\subsection{Effectiveness analysis of PPFC-GAN}
Two state-of-the-art FC methods, k-FED \cite{dennis2021heterogeneity} and federated fuzzy c-means (FFCM) \cite{stallmann2022towards}, are used to validate the effectiveness of PPFC-GAN. We also conduct ablation analysis to gain a better understanding of the method, i.e., performing  dimensionality reduction SAE and KM clustering sequentially, named as PPFC-GAN$^\dag$.

The numerical results of NMI \cite{strehl2002cluster} and kappa \cite{liu2019evaluation} are shown in Table \ref{NMI} and Table \ref{Kappa}.  One can see that: 1) Both metrics indicate that dimensionality reduction, SAE, can improve the clustering performance significantly, and the performance can be even better by the joint model, PPFC-GAN. For example, on the dataset MNIST with $p = 1$, PPFC-GAN$^\dag$ and PPFC-GAN improve k-FED by about 0.34 and 0.42 based on NMI respectively. And those based on kappa are about 0.41 and 0.45 respectively. 2) With the help of synthetic dataset, the proposed method demonstrates greater efficacy and robustness compared to the baselines, and can even benefit from non-IID scenarios. In particular, it exhibits superior performance compared to the centralized counterpart (DCN) when $p$ is large on MNIST and Fashion-MNIST. Actually, for each client, a higher value of $p$ corresponds to a greater proportion of samples from the same cluster, meaning that they are more similar with each other and the local GAN is easier to be trained. As a result, the cluster structuer of the global synthetic dataset is more pronounced, sometimes even clearer than that of the real dataset, which can be seen in Fig. \ref{tsne_latent}. It is noteworthy that a large $p$ is in line with the reality, as the preference of most people is focused on a single category/cluster. 3) The two evaluation metrics, NMI and kappa, yield significantly different rankings. The NMI values suggest that the joint model, PPFC-GAN, can enhance the performance in most cases, whereas the kappa values indicate that the performance of PPFC-GAN is overestimated and sometimes even regress. Although kappa discourages the proposed method, we must honestly point out that it is a more reliable metric than NMI. As shown in Fig. \ref{case_kappa}, there are two partitions $\pi^{(1)}$ and $\pi^{(2)}$ obtained by PPFC-GAN$^\dag$ and PPFC-GAN, respectively. Obviously, the label distribution of $\pi^{(1)}$ is more close to the ground-truth one. However, the NMI of $\pi^{(1)}$ and $\pi^{(2)}$ are 0.6812 and 0.7179 respectively, which is unreasonable. On the contrary, the rank obtained by kappa is more reasonable, 0.7403 for $\pi^{(1)}$ and 0.6966 for $\pi^{(2)}$. A more comprehensive discussion and more evidence of the disadvantages of NMI can be found in \cite{liu2019evaluation,yan2022selective}. 4) We further narrow the gap between federated clustering and centralized clustering by a big margin.

Although we have validated that dimensionality reduction is beneficial to clustering performance, an intuitive understanding of how it works in the clustering process is still lacking. Hence, we visualize the data distribution in both the original and latent data space using t-SNE \cite{van2008visualizing} in Fig. \ref{tsne_latent}. From the figure, one can see that: 1) There are many data points from different clusters mixed together. 2) The problem can be alleviated by PPFC-GAN$^\dag$ and PPFC-GAN. 3) The global synthetic dataset shows a clearer cluster structure than the global real dataset. These are why PPFC-GAN is more effective than the federated baselines, and even superior to DCN in centralized setting when $p$ = 1.

\begin{figure*}[!t]
\centering
\includegraphics[height = 8cm, width = 13.75cm]{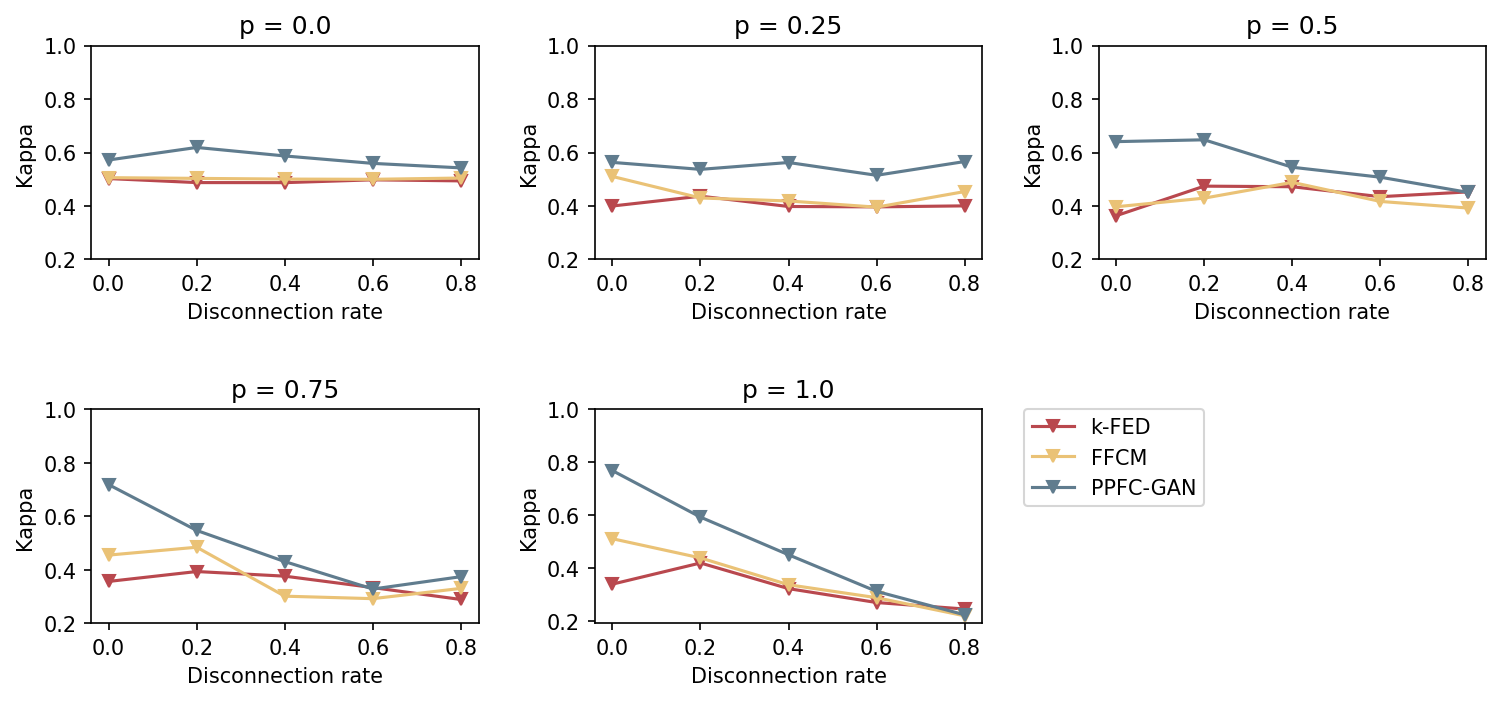}
\caption{The relations between the clustering performance and the device failures on the dataset MNIST.}
\label{case_ft}
\end{figure*}

\subsection{Sensitivity analysis of clustering performance to device failures}

During the training process, some client devices may lose connection with the server due to wireless network fluctuations, energy constraints, etc. Consequently, some specific data characteristics of the failed devices may be lost, resulting in poor and unrobust performance. Hence, it is critical for a federated model to be resilient to device failures.

To simulate different disconnected scenarios, we define \textbf{disconnection rate} that measures the percentage of the failed devices among all devices. In all disconnected scenarios, following \cite{li2020federated}, the federated model simply ignores the failed devices and continues training with the remaining ones. As shown in Fig. \ref{case_ft}, one can observe that: 1) The proposed method is almost always superior to k-FED and FFCM by a big margin. 2) The sensitivity of clustering performance to device failures is positively correlated with the non-IID level $p$, i.e., device failures affect the clustering performance more severely as $p$ increases. This is because the level of complementarity among clients decreases with a higher value of $p$, and in the extreme case where $p = 1$, there is no complementarity among clients.

In summary: 1) Dimensionality reduction can improve the clustering performance significantly. 2) The proposed method is more effective and robust than the baselines in immunizing the non-IID problem and the device failures, and can even benefit from some non-IID scenarios. 3) The sensitivity of clustering performance to device failures is positively correlated with the non-IID level $p$. 4) Kappa is a more reliable metric than NMI.

\section{Conclusion}
\label{sect6}
In this study, we introduce Privacy-Preserving Federated Clustering with GAN-Generated Samples (PPFC-GAN). Our method addresses a critical challenge posed by non-independent and non-identically-distributed (non-IID) data across clients, which has been well-documented in previous research. By leveraging synthetic data produced by local Generative Adversarial Networks (GANs), our approach mitigates the non-IID issue. The proposed framework harnesses deep clustering models to estimate global cluster centroids, which are subsequently communicated to clients for precise sample labeling. Crucially, we establish a theoretical foundation that underscores the inherent privacy guarantees encapsulated within the GAN-generated samples shared among clients. Our thorough experiments underscores the superior efficacy and robustness of the proposed method compared to baseline techniques.

While the amalgamation of the clustering model, Deep Clustering Network (DCN), may not be perfectly tailored to handle the clustering of the synthetic global dataset, a promising avenue for future research lies in the exploration of an end-to-end joint framework that seamlessly integrates both steps. Such a framework holds potential to enhance the overall clustering performance. Finally, mounting evidence suggests that kappa is a more dependable metric than NMI, and therefore, it is recommended to utilize kappa over NMI for evaluating the clustering performance.




\bibliographystyle{elsarticle-num}
\bibliography{references.bib}


\appendix
\section{Detailed hyperparameter settings}

The proposed method consists of two main steps, global synthetic data construction and cluster assignment. In the first step, we train $m$ local GANs with the local data, where $m$ is the number of clients and is set to the number of true clusters \cite{stallmann2022towards}. All of these local GANs have the same hyperparameter settings, and each one comprises two networks: the generator and discriminator. In the second step, we train a stacked autoencoder (SAE) on the  constructed global synthetic dataset. The detailed hyperparameter settings of these networks are tuned according to the specific dataset. For GAN, the hyperparameter settings are shown in Tables \ref{GAN_mnist}-\ref{GAN_pen}. For SAE, one can find those in Table \ref{AE}.

\begin{table*}[!h]
\center
\caption{GAN hyperparameters for MNIST and Fashion-MNIST. BN stands for batch normalization.}
\renewcommand{\arraystretch}{1.25} 
\tabcolsep 2mm 
\begin{tabular}{lcccccc}
\hline\hline
Operation &Kernel &Stride &Padding &Layer width / Feature maps &BN? &Nonlinearity \\\hline
Generator (Input 72 dim vector)\\
Linear                      &-&-&- &1024 &$\surd$ &ReLU\\
Linear                      &-&-&- &6272 &$\surd$ &ReLU\\
Transposed Convolution      &4&2&1 &64 &$\surd$ &ReLU\\
Transposed Convolution      &4&2&1 &1 &- &Sigmoid\\\hline

Discriminator (Input 28 $\times$ 28 $\times$ 1)\\
Convolution                 &4&2&1 &64   &$\surd$ &LeakyReLU\\
Convolution                 &4&2&1 &128  &$\surd$ &LeakyReLU\\
\underline{Convolution} (Removed for Fashion-MNIST)  &4&2&1 &256  &$\surd$ &LeakyReLU\\
Linear                      &-&-&- &1024 &$\surd$ &LeakyReLU\\
Linear                      &-&-&- &1    &- &-\\\hline
Generator Optimizer &\multicolumn{6}{c}{Adam $(lr = \num{6e-4},\, \beta_1 = 0.5,\, \beta_2 = 0.999)$} \\
Discriminator Optimizer &\multicolumn{6}{c}{Adam $(lr = \num{2e-4},\, \beta_1 = 0.5,\, \beta_2 = 0.999)$}\\
Batch size &\multicolumn{6}{c}{64}\\
Leaky ReLU slope &\multicolumn{6}{c}{0.2}\\
\hline\hline
\label{GAN_mnist}
\end{tabular}
\end{table*}

\begin{table*}[!t]
\caption{GAN hyperparameters for CIFAR-10 and STL-10. BN stands for batch normalization.}
\renewcommand{\arraystretch}{1.25} 
\tabcolsep 4.55mm 
\begin{tabular}{lcccccc}
\hline\hline
Operation &Kernel &Stride &Padding &Feature maps &BN? &Nonlinearity \\\hline
Generator (Input 1 $\times$ 1 $\times$ 72)\\
Transposed Convolution      &4&1&0 &128 &$\surd$ &ReLU\\
Transposed Convolution      &4&2&1 &64  &$\surd$ &ReLU\\
Transposed Convolution      &4&2&1 &32  &$\surd$ &ReLU\\
Transposed Convolution      &4&2&1 &3   &-       &Tanh\\\hline

Discriminator (Input 32 $\times$ 32 $\times$ 3)\\
Convolution                 &4&2&1 &32   &-  &LeakyReLU\\
Convolution                 &4&2&1 &64   &$\surd$  &LeakyReLU\\
Convolution                 &4&2&1 &128  &$\surd$  &LeakyReLU\\
Convolution                 &4&1&0 &1    &- &-\\\hline
Generator Optimizer &\multicolumn{6}{c}{Adam $(lr = \num{6e-4},\, \beta_1 = 0.5,\, \beta_2 = 0.999)$} \\
Discriminator Optimizer &\multicolumn{6}{c}{Adam $(lr = \num{2e-4},\, \beta_1 = 0.5,\, \beta_2 = 0.999)$}\\
Batch size &\multicolumn{6}{c}{64}\\
Leaky ReLU slope &\multicolumn{6}{c}{0.2}\\
\hline\hline
\label{GAN_cf10}
\end{tabular}
\end{table*}

\begin{table*}[!t]
\caption{GAN hyperparameters for Pendigits. BN stands for batch normalization.}
\renewcommand{\arraystretch}{1.25} 
\tabcolsep 11.5mm 
\begin{tabular}{lccc}
\hline\hline
Operation &Layer width &BN? &Nonlinearity \\\hline
Generator (Input 15 dim vector)\\
Linear    &256 &$\surd$ &LeakyReLU\\
Linear    &256 &$\surd$ &LeakyReLU\\
Linear    &16  &- &Sigmoid\\\hline

Discriminator (Input 16 dim vector)\\
Linear       &256 &$\surd$ &LeakyReLU\\
Linear       &256 &$\surd$ &LeakyReLU\\
Linear       &1   &-       &-\\\hline
Generator Optimizer &\multicolumn{3}{c}{Adam $(lr = \num{6e-4},\, \beta_1 = 0.5,\, \beta_2 = 0.999)$} \\
Discriminator Optimizer &\multicolumn{3}{c}{Adam $(lr = \num{2e-4},\, \beta_1 = 0.5,\, \beta_2 = 0.999)$}\\
Batch size &\multicolumn{3}{c}{64}\\
Leaky ReLU slope &\multicolumn{3}{c}{0.2}\\
\hline\hline
\label{GAN_pen}
\end{tabular}
\end{table*}

\begin{table*}[!t]
\caption{SAE hyperparameters. BN stands for batch normalization.}
\renewcommand{\arraystretch}{1.5} 
\tabcolsep 12.5mm 
\begin{tabular}{lccc}
\hline\hline
Operation &Layer width &BN? &Nonlinearity \\\hline
Encoder (Input $z_c$ dim vector)\\
Linear    &500  &-       &ReLU\\
Linear    &500  &$\surd$ &ReLU\\
Linear    &2000 &$\surd$ &ReLU\\
Linear    &10   &$\surd$ &ReLU\\\hline

Decoder (Input 10 dim vector) \\
Linear       &2000   &$\surd$ &ReLU\\
Linear       &500    &$\surd$ &ReLU\\
Linear       &500    &$\surd$ &ReLU\\
Linear       &$z_c$  &-       &-\\\hline
Optimizer &\multicolumn{3}{c}{Adam $(lr = \num{2e-3},\, \beta_1 = 0.9,\, \beta_2 = 0.999)$} \\
Batch size &\multicolumn{3}{c}{$0.01 \times n$}\\
Leaky ReLU slope &\multicolumn{3}{c}{0.2}\\
\hline\hline\\
\multicolumn{4}{l}{where $z_c$ = 784 for MNIST and Fashion-MNIST,  3072 for CIFAR-10 and STL-10, 16 for Pendigits. $n$ is the nu-}\\
\multicolumn{4}{l}{mber of samples in the constructed global synthetic dataset.}
\label{AE}
\end{tabular}
\end{table*}

\,
\newpage
\,
\newpage

\end{sloppypar}
\end{document}